%% file: lpr.tex
\pdfoutput=1

\documentclass{svjour3}                     
\smartqed  
\usepackage{graphicx}

\usepackage{mathptmx}      
\usepackage{amsmath}
\usepackage{amsfonts}
\usepackage{natbib}
\usepackage{latexsym}
\newcommand{\norm}[1]{\left\Vert#1\right\Vert}

\newcommand{\tr}[1]{\textrm{tr}\left(#1\right)}
\newcommand{\Real}{\mathbb{R}}
\newcommand{\R}{\mathbb{R}}
\newcommand{\eps}{\varepsilon}

\newcommand{\df}[1]{\frac{\partial}{\partial#1}}

\newcommand{\argmax}{\operatornamewithlimits{argmax}}

\newcommand{\OO}{\mathcal{O}}
\newcommand{\vp}{v_p^{\perp}}
\newcommand{\ap}{\alpha_p}
\newcommand{\bp}{\beta_p}

\newcommand{\xj}{(x_{i_j}-\bar{x}_i)}
\newcommand{\zj}{(z_{i_j}-\bar{z}_i)}
\newcommand{\sjk}{\sum_{j=1}^{k(i)}}
\newcommand{\spd}{\sum_{p=1}^{d}}

\journalname{Machine Learning }
\begin{document}

\title{Local Procrustes for Manifold Embedding:\\ A Measure of Embedding Quality and Embedding
Algorithms
\thanks{This research was supported in part by Israeli Science
Foundation grant.} }

\titlerunning{Local Procrustes}        

\author{Yair Goldberg         \and
        Ya'acov Ritov
}


\institute{Y. Goldberg
\at  Department of Statistics, The Hebrew University, 91905 Jerusalem, Israel\\
              \email{yair.goldberg@mail.huji.ac.il}           
           \and
           Y. Ritov\\
\email{yaacov.ritov@huji.ac.il} }
\date{Received: date / Accepted: date}

\maketitle

\begin{abstract}
We present the Procrustes measure, a novel measure based on
Procrustes rotation that enables quantitative comparison of the
output of manifold-based embedding algorithms (such as
LLE~\citep{LLE} and Isomap~\citep{ISOMAP}). The measure also serves
as a natural tool when choosing dimension-reduction parameters. We
also present two novel dimension-reduction techniques that attempt
to minimize the suggested measure, and compare the results of these
techniques to the results of existing algorithms. Finally, we
suggest a simple iterative method that can be used to improve the
output of existing algorithms.

 \keywords{Dimension reducing \and Manifold learning \and Procrustes analysis, \and Local
PCA \and Simulated annealing}
\end{abstract}

\input{intro}

\input{problem}
\input{faithful}
\input{algo}
\input{numerical}
\input{discussion}
\appendix
\section{Proofs}

\input{appendix}

\begin{acknowledgements}
We would like to thank S. Kirkpatrick  and J. Goldberger for
meaningful discussions. We are grateful to the anonymous reviewers
of an earlier version of this manuscript for their helpful
suggestions.
\end{acknowledgements}


\end{document}

%% file: intro.tex
\section{Introduction}

Technological advances constantly improve our ability to collect and
store large sets of data. The main difficulty in analyzing such
high-dimensional data sets is, that the number of observations
required to estimate functions at a set level of accuracy grows
exponentially with the dimension. This problem, often referred to as
the curse of dimensionality, has led to various techniques that
attempt to reduce the dimension of the original data.

Historically, the main approach to dimension reduction is the linear
one. This is the approach used by principle component analysis (PCA)
and factor analysis \citep[see][for both]{Mardia}. While these
algorithms are largely successful, the assumption that a linear
projection describes the data well is incorrect for many data sets.
A more realistic assumption than that of an underlying linear
structure is that the data is on, or next to, an embedded manifold
of low dimension in the high-dimensional space. Here a manifold is
defined as a topological space that is locally equivalent to a
Euclidean space. Locally, the manifold can be estimated by linear
approximations based on small neighborhoods of each point. Many
algorithms were developed to perform embedding for manifold-based
data sets, including the algorithms suggested
by~\citet{LLE,ISOMAP,belkin,LTSA,HessianEigenMap,Weinberger,dollar}.
Indeed, these algorithms have been shown to succeed even where the
assumption of linear structure does not hold. However, to date there
exists no good tool to estimate the quality of the result of these
algorithms.

Ideally, the quality of an output embedding could be judged based on
a comparison to the structure of the original manifold. Indeed, a
measure based on the idea that the manifold structure is known to a
good degree was recently suggested by~\citet{dollar}. However, in
the general case, the manifold structure is not given, and is
difficult to estimate accurately. As such ideal measures of quality
cannot be used in the general case, an alternate quantitative
measure is required.

In this work we suggest an easily computed function that measures
the quality of any given embedding. We believe that a faithful
embedding is an embedding that preserves the structure of the local
neighborhood of each point. Therefore the quality of an embedding is
determined by the success of the algorithm in preserving these local
structures. The function we present, based on the Procrustes
analysis, compares each neighborhood in the high-dimensional space
and its corresponding low-dimensional embedding. Theoretical results
regarding the convergence of the proposed measure are presented.

We further suggest two new algorithms for discovering the
low-dimensional embedding of a high-dimensional data set, based on
minimization of the suggested measure function. The first
algorithm performs the embedding one neighborhood at a time. This
algorithm is extremely fast, but may suffer from incremental
distortion. The second algorithm, based on simulated
annealing~\citep{kirkpatrick}, performs the embedding of all local
neighborhoods simultaneously. A simple iterative procedure that
improves on an existing output is also presented.

The paper is organized as follows. The problem of manifold learning
is presented in Section~\ref{sec:problem}. A discussion regarding
the quality of embeddings in general and the suggested measure of
quality are presented in Section~\ref{sec:Faithful}. The embedding
algorithms are presented in Section~\ref{sec:algo}. In
Section~\ref{sec:NumericalExamples} we present numerical examples.
All proofs are presented in the Appendix.

%% file: problem.tex
\section{Manifold-learning problem setting and definitions}\label{sec:problem}
In this section we provide a formal definition of the
manifold-learning dimension-reduction problem.

Let $\mathcal{M}$ be a $d$-dimensional manifold embedded in
$\Real^q$. Assume that a sample is taken from $\mathcal{M}$. The
goal of manifold-learning is to find a \emph{faithful} embedding of
the sample in $\Real^d$. The assumption that the sample is taken
from a manifold is translated to the fact that small distances on
the manifold $\mathcal{M}$ can be approximated well by the Euclidian
distance in $\Real^q$. Therefore, to find an embedding, one first
needs to approximate the structure of small neighborhoods on the
manifold using the Euclidian metric in $\Real^q$. Then one must find
a unified embedding of the sample in $\R^d$ that preserves the
structure of local neighborhoods on $\mathcal{M}$.

In order to adhere to this scheme, we need two more assumptions.
First, we assume that $\mathcal{M}$ is \emph{isometrically
embedded} in $\R^q$. By definition, an isometric mapping between
two manifolds preserves the inner product on the tangent bundle at
each point. Less formally, this means that distances and angles
are preserved by the mapping. This assumption is needed because we
are interested in an embedding that everywhere preserves the local
structure of distances and angles between neighboring points. If
this assumption does not hold, close points on the manifold may
originate from distant points in $\R^d$ and vice versa. In this
case, the structure of the local neighborhoods on the manifold
will not reveal the structure of the original neighborhoods in
$\R^d$. We remark here that the assumption that the embedding is
isometric is strong but can be relaxed. One may assume instead
that the embedding mapping is conformal. This means that the inner
products on the tangent bundle at each point are preserved up to a
scalar $c$ that may change continuously from point to point. Note
that the class of isometric embeddings is included in the class of
conformal embeddings. While our main discussion regards isometric
embeddings, we will also discuss the conformal embedding problem,
which is the framework of algorithms such as
c-Isomap~\citep{cIsomap} and Conformal Embeddings
(CE)~\citep{ShaExtensionSpectralMethods}.

The second assumption is that the sample taken from the manifold
$\mathcal{M}$ is \emph{dense}. We need to prevent the situation in
which the local neighborhood of a point, which is computed according
to the Euclidian metric in $\R^q$, includes distant geodesic points
on the manifold. This can happen, for example, if the manifold is
twisted. The result of having distant geodesic points in the same
local neighborhood is that these distant points will be embedded
close to each other instead of preserving the true geodesic distance
between them.

To define the problem formally, we require some definitions.

The \emph{neighborhood} of a point $x_i\in\mathcal{M}$ is a set of
points $X_i$ that consists of points close to $x_i$ with respect
to the Euclidean metric in $\R^q$. For example, neighbors can be  $K$-nearest neighbors or all the points in an $\eps$-ball around $x_i$.

The \emph{minimum radius of curvature}
$r_0=r_0(\mathcal{M})$ is defined as follows:
\begin{equation*}
    \frac{1}{r_0}=\max_{\gamma,t} \left\{ \norm{\ddot{\gamma(t)}}\right\}
\end{equation*}
where $\gamma$ varies over all unit-speed geodesics in
$\mathcal{M}$ and $t$ is in a domain of $\gamma$.

The \emph{minimum branch separation} $s_0=s_0(\mathcal{M})$ is
defined as the largest positive number for which
$\norm{x-\tilde{x}}<s_0$ implies $d_\mathcal{M}(x,\tilde{x})\leq
\pi r_0$, where $x,\tilde{x}\in \mathcal{M}$ and
$d_\mathcal{M}(x,\tilde{x})$ is the geodesic distance between $x$
and $\tilde{x}$~\citep[see][for both
definitions]{IsoMapConvergence}.

We define the radius $r(i)$ of neighborhood $i$ to be
\begin{equation*}
r(i)= \max_{j \in \{1,\ldots,k(i)\}} \norm{x_i-x_{i_j}}
\end{equation*}
where $x_{i_j}$ is the $j$-th out of the $k(i)$ neighbors of $x_i$.
Finally, we define $r_{\max}$ to be the maximum over $r(i)$ .

We say that the sample is \emph{dense} with respect to the chosen
neighborhoods if $r_{\max}<s_0$. Note that this condition depends on
the manifold structure, the given sample, and the choice of
neighborhoods. However, for a given compact manifold, if the
distribution that produces the sample is supported throughout the
entire manifold, this condition is valid with probability increasing
towards $1$ as the size of the sample is increased and the radius of
the neighborhoods is decreased.

We now state the manifold-learning problem more formally. Let
$D\subset \R^d$ be a compact set and let $\phi:D\rightarrow \R^q$
be a smooth and invertible isometric mapping. Let $\mathcal{M}$ be
the $d$-dimensional image of $D$ in $\R^q$. Let $x_1,\ldots,x_n$
be a sample taken from $\mathcal{M}$. Define neighborhoods $X_i$
for each of the points $x_i$. Assume that the sample
$x_1,\ldots,x_n$ is \emph{dense} with respect to the choice of
$X_i$. Find $y_1,\ldots,y_n \in \R^d$ that approximate
$\phi^{-1}(x_1),\ldots,\phi^{-1}(x_n)$ up to rotation and
translation.

%% file: faithful.tex
\section{Faithful embedding}\label{sec:Faithful}
As discussed in Section~\ref{sec:problem}, a \emph{faithful}
embedding should preserve the structure of local neighborhoods on
the manifold, while finding a global embedding mapping.

In this section, we will attempt to answer the following two questions:
\begin{enumerate}
    \item How do we define ``preservation of the local structure of a neighborhood"?
    \item How do we find a global embedding that preserves the local structure?
\end{enumerate}

We now address the first question. Under the assumption of isometry,
it seems reasonable to demand that neighborhoods on the manifold and
their corresponding embeddings be closely related. A neighborhood on
the manifold and its embedding can be compared using the
\emph{Procrustes statistic}. The Procrustes statistic measures the
distance between two configurations of points. The statistic
computes the sum of squares between pairs of corresponding points
after one of the configurations is rotated and translated to best
match the other.

In the remainder of this paper we will represent any set of $k$
points $x_1,\ldots,x_k\in\R^q$ as a matrix $X_{k\times
q}=[x_1',\ldots,x_k']$; i.e., the $j$-th row of the matrix $X$
corresponds to $x_j$.

 Let $X$ be a
$k$-point set in $\R^q$ and let $Y$ be a $k$-point set in $\R^d$,
where $d \leq q$. We define the \emph{Procrustes statistic} $G(X,Y)$
as
\begin{eqnarray}\label{eq:G}
    G(X,Y)&=&\inf_{\{A,b: \,A'A=I,\,b\in \R^q\}} \sum_{i=1}^k \norm{x_i-Ay_i-b}^2 \\
    &=&\inf_{\{A,b:\,A'A=I,\,b\in \R^q\}}
    \tr{(X-YA'-1b')'(X-YA'-1b')}\nonumber
\end{eqnarray}
where the rotation matrix $A$ is a columns-orthogonal $q\times d$
matrix, $A'$ is the adjoint of $A$, and $1$ is a $k$-dimensional
vector of ones.

The \emph{Procrustes rotation matrix} $A$ and the \emph{Procrustes
translation vector} $b$ that minimize $G(X,Y)$ can be computed
explicitly, as follows. Let $Z=X'HY$ where $H \equiv
I-\frac{1}{k}11'$ is the centering matrix. Let $ULV'$ be the
singular-value decomposition (svd) of $Z$, where $U$ is an
orthogonal $q\times d$ matrix, $L$ is a non-negative $d\times d$
diagonal matrix, and $V$ is a $d\times d$ orthogonal
matrix~\citep{Mardia}. Then, the \emph{Procrustes rotation matrix}
$A$ is given by $UV'$~\citep{sibson}. The \emph{Procrustes
translation vector} $b$ is given by $\bar{x}-A\bar{y}$, where
$\bar{x}$ and $\bar{y}$ are the sample means of $X$ and $Y$,
respectively. Due to the last fact, we may write $G(X,Y)$ without
the translation vector $b$ as

\begin{equation*}
    G(X,Y)=\inf_{\{A:\,A'A=I\}}
    \tr{(X-YA')'H(X-YA')}=\inf_{\{A:\,A'A=I\}}\norm{H(X-YA')}_F^2\,,
\end{equation*}
where $\norm{\,}_F$ is the Frobenius norm.

Given $X$, the minimum of $G(X,Y)$ can be computed explicitly and
is achieved by the first $d$ principal components of $X$. This
result is a consequence of the following lemma.
\begin{lemma}\label{lem:PCAminimzeR}
Let $X=X_{k\times q}$ be a centered matrix of rank $q$ and let
$d\leq q$. Then
\begin{equation}
\inf_{\{\tilde{X}:\,\mathrm{rank}(\tilde{X})=d\}}\|X-\tilde{X}\|^2_F
\end{equation}
is obtained when $\tilde{X}$ equals the projection of $X$ on the
subspace spanned by the first $d$ principal components of $X$.
\end{lemma}
For proof, see~\citet[][page 220]{Mardia}.

Returning to the questions posed at the beginning of this section,
we will define how well an embedding preserves the local
neighborhoods using the Procrustes statistic $G(X_i,Y_i)$ of each
neighborhood-embedding pair $(X_i,Y_i)$. Therefore, a global
embedding that preserves the local structure can be found by
minimizing the sum of the Procrustes statistics of all
neighborhood-embedding pairs.

More formally, let $X$ be the $q$-dimensional sample from the
manifold and let $Y$ be a $d$-dimensional embedding of $X$. Let
$X_i$ be the neighborhood of $x_i$ ($i=1,\dots,n$) and $Y_i$ its
embedding. Define
\begin{equation}\label{eq:R}
R(X,Y)=\frac{1}{n}\sum_{i=1}^n G(X_i,Y_i)\;.
\end{equation}
The function $R$ measures the average quality of the neighborhood
embeddings. Embedding $Y$ is considered better than embedding
$\tilde{Y}$ in the local-neighborhood-preserving sense if
$R(X,Y)<R(X,\tilde{Y})$. This means that on the average, $Y$
preserves the structure of the local neighborhoods better than
$\tilde{Y}$.

The function $R(X,Y)$ is sensitive to scaling, therefore
normalization should be considered. A reasonable normalization is
\begin{equation}\label{eq:Rnormalized}
    R_{N}(X,Y)=\frac{1}{n}\sum_{i=1}^n
    G(X_i,Y_i)/ \norm{HX_i}_F^2 \,.
\end{equation}
The $i$-th summand of $R_{N}(X,Y)$ examines how well the rotated and
translated $Y_i$ ``explains" $X_i$, independent of the size of
$X_i$. This normalization solves the problem of increased weighting
for larger neighborhoods that exists in the unnormalized $R(X,Y)$.
It also allows comparison of embeddings for data sets of different
sizes. Hence, this normalized version is used to compare the results
of different outputs (see Section~\ref{sec:NumericalExamples}).


In the remainder of this section, we will justify our choice of the
Procrustes measure $R$ for a quantitative comparison of embeddings.
We will also present two additional, closely related measures. One
measure is $R_{PCA}$, which can ease computation when the input
space is of high dimension. The second measure is $R_{C}$, which is
a statistic designed for conformal mappings (see
Section~\ref{sec:problem}). Finally, we will discuss the relation
between the measures suggested in this work to the objective
functions of other algorithms, namely LTSA~\citep{LTSA} and
SDE~\citep{Weinberger}.

We now justify the use of the Procrustes statistic $G(X_i,Y_i)$ as a
measure of the quality of the local embedding. First, $G(X_i,Y_i)$
estimates the relation between the entire input neighborhood and its
embedding as one entity, instead of comparing angles and distances
within the neighborhood with those within its embedding. Second, the
Procrustes statistic is not highly sensitive to small perturbations
of the embedding. More formally, $G(X,Y)=\OO(\eps^2)$, where
$Y=X+\eps Z$ and $Z$ is a general
matrix~\citep[see][]{SibsonRobustness}. Finally, the function $G$ is
$l_2$-norm-based and therefore prefers small differences at many
points to big differences at fewer points. This is preferable in our
context, as the local embedding of the neighborhood should be
compatible with the embeddings of nearby neighborhoods.

The usage of $R$ as a measure of the quality of the global embedding
of the manifold is justified by Theorem~\ref{thm:R convergence}.
Theorem~\ref{thm:R convergence} claims that when the number of input
points $X$ increases, the low-dimensional points $Z=\phi^{-1}(X)$ of
the input data tend to zero $R$. This implies that the minimizer $Y$
of $R$ should be close to the original data set $Z$ (up to rotation
and translation).

\begin{theorem}\label{thm:R convergence}
Let $\mathcal{D}$ be a compact connected set. Let
$\phi:\mathcal{D}\rightarrow\R^q$ be an isometry. Let $X^{(n)}$ be
an $n$-point sample taken from $\phi(\mathcal{D})$, and let
$Z^{(n)}=\phi^{-1}(X^{(n)})$. Assume that the sample $X^{(n)}$ is
dense with respect to the choice of neighborhoods for all $n \ge
N_0$. Then for all $n \ge N_0$
\begin{equation}\label{eq:thRfirst}
R\left(X^{(n)},Z^{(n)}\right)=\OO(r_{max}^4)\,.
\end{equation}
\end{theorem}
See Appendix~\ref{sec:proof1} for proof.

Replacing $R(X,Y)$ with the normalized version $R_{N}(X,Y)$ (see
Eq.~\ref{eq:Rnormalized}) and noting that
$\norm{HX_i}_F^2=\OO(r_i^2)$, we obtain
\begin{corollary}\label{cor:Rnormalized}
\begin{equation*}
R_{N}(X^{(n)},Z^{(n)})= \OO(r_{max}^2)\,.
\end{equation*}

\end{corollary}

To avoid heavy computations, a slightly different version of
$R(X,Y)$ can be considered. Instead of measuring the difference
between the original neighborhoods on the manifold and their
embeddings, one can compare the local PCA projections~\citep{Mardia}
of the original neighborhoods with their embeddings. We therefore
define
\begin{equation}\label{eq:RPCA}
    R_{PCA}(X,Y)=\frac{1}{n}\sum_{i=1}^n G(X_iP_i,Y_i)\,,
\end{equation}
where $P_i$ is the $d$-dimensional PCA projection matrix of $X_i$.

The convergence theorem for $R_{PCA}$ is similar to
Theorem~\ref{thm:R convergence}, but the convergence is slower.

\begin{theorem}\label{thm:RPCA convergence}
Let $\mathcal{D}$ be a compact connected set. Let
$\phi:\mathcal{D}\rightarrow\R^q$ be an isometry. Let $X^{(n)}$ be
an $n$-point sample taken from $\phi(\mathcal{D})$, and let
$Z^{(n)}=\phi^{-1}(X^{(n)})$. Assume that the sample $X^{(n)}$ is
dense with respect to the choice of neighborhoods for all $n \ge
N_0$. Then for all $n \ge N_0$
\begin{equation*}
R_{PCA}\left(X^{(n)},Z^{(n)}\right)=\OO(r_{max}^3)\,,
\end{equation*}
\end{theorem}
See Appendix~\ref{sec:proof2} for proof.

We now present another version of the Procrustes measure
$R_{C}(X,Y)$, suitable for conformal mappings. Here we want to
compare between each original neighborhood $X_i$ and its
corresponding embedding $Y_i$, where we allow $Y_i$ not only to be
rotated and translated but also to be rescaled. Define
\begin{equation*}
    G_{C}(X,Y)=\inf_{\{A:\,A'A=I,0<c\in \Real\}}\tr{\left(X-Y(cA')\right)'H\left(X-Y(cA')\right)}\,.
\end{equation*}
Note that the scalar $c$ was introduced here to allow scaling of
$Y$.  Let $Z=X'HY$ and let $ULV'$ be the svd of $Z$. The minimizer
rotation matrix $A$ is given by $UV'$. The minimizer constant $c$
is given by $\tr{L}/\tr{Y'Y}$\citep[see][]{sibson}. The
(normalized) conformal Procrustes measure is given by
\begin{equation}\label{eq:RC}
    R_{C}(X,Y)=\frac{1}{n}\sum_{i=1}^n G_{C}(X_i,Y_i)/ \norm{HX_i}_F^2 \,.
\end{equation}
Note that $R_{C}(X,Y)\leq R_{N}(X,Y)$ since the constraints are
relaxed in the definition of $R_{C}(X,Y)$. However, the lower
bound in both cases is equal (see Lemma~\ref{lem:PCAminimzeR}).

We present a convergence theorem, similar to that of $R$ and
$R_{PCA}$.
\begin{theorem}\label{thm:RC convergence}
Let $\mathcal{D}$ be a compact connected set. Let
$\tilde{\phi}:\mathcal{D}\rightarrow\R^q$ be a conformal mapping.
Let $X^{(n)}$ be an $n$-point sample taken from
$\tilde{\phi}(\mathcal{D})$, and let
$Z^{(n)}=\tilde{\phi}^{-1}(X^{(n)})$. Assume that the sample
$X^{(n)}$ is dense with respect to the choice of neighborhoods for
all $n \ge N_0$. Then for all $n \ge N_0$ we have\begin{equation*}
R_{C}\left(X^{(n)},Z^{(n)}\right)=\OO(r_{max}^2)\,.
\end{equation*}
\end{theorem}
See Appendix~\ref{sec:proof3} for proof. Note that this result is of
the same convergence rate as of $R_{N}$ (see
Corollary~\ref{cor:Rnormalized}).

A cost function somewhat similar to the measure $R_{PCA}$ was
presented by~\citet{LTSA} in a slightly different context. The local
PCA projection of neighborhoods was used as an approximation of the
tangent plane at each point. The resulting algorithm, local tangent
subspaces alignment (LTSA), is based on minimizing the function
\begin{equation*}
    \sum_{i=1}^n \norm{(I-P_iP_i')HY_i}_F^2\,,
\end{equation*}
where the $k(i) \times d$ matrices $P_i$ are as in $R_{PCA}$. The
minimization performed by LTSA is under a normalization
constraint. This means that as a measure, LTSA's objective
function is designed to compare normalized outputs $Y$ (otherwise
$Y=0$ would be the trivial minimizer) and is therefore unsuitable
as a measure.

Another algorithm worth mentioning here is SDE~\citep{Weinberger}.
The constraints on the output required by this algorithm are that
all the distances and angles within each neighborhood be
preserved. Therefore the output of this algorithm should always be
close to the minimum of $R(X,Y)$. The maximization of the
objective function of this algorithm
\begin{equation*}
 \sum_{i,j} \norm{y_i-y_j}^2
\end{equation*}
is reasonable when the aforementioned constraints are enforced.
However, it is not relevant as a measure for comparison of general
outputs that do not fulfill these constraints.

In summary, we return to the questions posed at the beginning of
this section. We choose to define preservation of local
neighborhoods as minimization of the Procrustes measure $R$ (see
Eq.~\ref{eq:R}). We therefore find a global embedding that best
preserves local structure by minimizing $R$. For computational
reasons, minimization of $R_{PCA}$ (see Eq.~\ref{eq:RPCA}) may be
preferred. For conformal maps we suggest the measure $R_{C}$ (see
Eq.~\ref{eq:RC}), which allows separate scaling of each
neighborhood.

%% file: algo.tex
\section{Algorithms}\label{sec:algo}
In Section~\ref{sec:Faithful} we showed that a faithful embedding
should yield low $R(X,Y)$ and $R_{PCA} (X,Y)$ values. Therefore we
may attempt to find a faithful embedding by minimizing these
functions. However, $R(X,Y)$ and $R_{PCA} (X,Y)$ are not necessarily
convex functions and may have more than one local minimum. In this
section we present two algorithms for minimization of $R(X,Y)$ or
$R_{PCA} (X,Y)$. In addition, we present an iterative method that
can improve the output of the two algorithms, as well as other
existing algorithms.

The first algorithm, greedy Procrustes (GP), performs the embeddings
one neighborhood at a time. The first neighborhood is embedded using
the PCA projection (see Lemma~\ref{lem:PCAminimzeR}). At each stage,
the following neighborhood is embedded by finding the best embedding
with respect to the embeddings already found. This greedy algorithm
is presented in Section~\ref{sec:greedy}.

The second algorithm, Procrustes subspaces alignment (PSA), is based
on an alignment of the local PCA projection subspaces. The use of
local PCA produces a good low-dimensional description of the
neighborhoods, but the description of each of the neighborhoods is
in an arbitrary coordinate system. PSA performs the global embedding
by finding the local embeddings and then aligning them. PSA is
described in Section~\ref{sec:aligmentMinimization}. Simulated
annealing (SA) is used to find the alignment of the subspaces (see
Section~\ref{sec:SA}).

After an initial solution is found using either GP or PSA, the
solution can be improved using an iterative procedure until there is
no improvement (see Section~\ref{sec:iterative}).

\subsection{Greedy Procrustes (GP)}\label{sec:greedy}
GP finds the neighborhoods' embeddings one by one.
The flow of the algorithm is described below.
\begin{enumerate}
    \item \textbf{Initialization:}\begin{itemize}
        \item Find the neighbors $X_i$ of each point $x_i$ and let $\mathrm{Neighbors}(i)$ be the indices of the neighbors $X_i$.
        \item Initialize the list of all embedded points' indices to $N:=\emptyset$.
    \end{itemize}

    \item \textbf{Embedding of the first neighborhood:} \begin{itemize}
        \item Choose an index $i$ (randomly).
        \item Calculate the embedding $Y_i=X_i P_i$, where $P_i$ is the PCA projection matrix of $X_i$.
        \item Update the list of indices of embedded points   $N=\mathrm{Neighbors}(i)$.
    \end{itemize}
    \item \textbf{Find all other embeddings (iteratively):}
    \begin{itemize}
        \item  Find $j$, where $x_j$ is the unembedded point with the largest number of embedded neighbors,\\
        $ j=\argmax_{p\notin N}\left|\{\mathrm{Neighbors}(p)
        \cap N\}\right|$.
        \item Define $\overline{X}_j=\left\{x_p|p\in \mathrm{Neighbors}(j)\cap N\right\}$,
        the points in $X_j$ that are already embedded.
        \\Define $\overline{Y}_j$ as the (previously calculated) embedding of $\overline{X}_j$.\\
        Define $\widetilde{X}_j=\left\{x_p|p\in \mathrm{Neighbors}(j)\setminus N\right\}$, the points in $X_j$ that are
         not embedded yet.
        \item Compute the Procrustes rotation matrix $A_j$ and
         translation vector $b_j$ between $\overline{X}_j$ and
         $\overline{Y}_j$.
        \item Define the embedding of the points in $\widetilde{X}_j$ as
        $\widetilde{Y}_j=\widetilde{X}_jA_j+b_j$.
        \item Update the list of indices of embedded points  $N$\\ $N:=N\cup \mathrm{Neighbors}(j)$.
\end{itemize}
\item \textbf{Stopping condition:}
        \\Stop when $|N|=n$, i.e., when all points are embedded.
\end{enumerate}

The main advantage of GP is that it is fast. The embedding for $X_i$
is computed in $\OO(k(i)^3)$ where $k(i)$ is the size of $X_i$.
Therefore, the overall computation time is $\OO(nK^3)$, where
$K=\max_i k(i)$. While GP does not claim to find the global minimum,
it does find an embedding that preserves the local neighborhood's
structure. The main disadvantage of GP is that it has incremental
errors.

\subsection{Procrustes Subspaces Alignment (PSA)}\label{sec:aligmentMinimization}
$R(X,Y)$ can be written in terms of the Procrustes rotation matrices
$A_i$ as
\begin{eqnarray*}
    R(X,Y)&=& \frac{1}{n}\sum_{i=1}^n \inf_{\{A_i:\,A_i'A_i=I\}}
    \norm{H(X_i  -Y_i A_i')}_F^2\,,
\end{eqnarray*}
where $H$ is the centering matrix. $A_i$ can be calculated, given
$X$ and $Y$. However, as $Y$ is not given, one way to find $Y$ is
by first guessing the matrices $A_i$ and then finding the $Y$ that
minimizes
\begin{equation}\label{eq:RXYA1AN}
    R(X,Y|A_1,\ldots,A_n)=\frac{1}{n}\sum_{i=1}^n \norm{H(X_i  -Y_i A_i')}_F^2\,.
\end{equation}
$Y$ can be found by taking derivatives of $R(X,Y|A_1,\ldots,A_n)$.

We therefore need to choose $A_i$ wisely. The choice of the Jacobian
matrices $J_i\equiv J_\phi (z_i)$ as a guess for $A_i$ is justified
by the following, as shown in the proof of Theorem~\ref{thm:R
convergence} (see Eq.~\ref{eq:proofR}). As the size of the sample is
increased, $\frac{1}{n}\sum_{i=i}^{n} \norm{H(X_i -Z_i
J_i')}\rightarrow 0$. This means that choosing $J_i$ will lead to a
solution $Y$ that is close to the minimizer of $R(X,Y)$.

The PCA projection matrices $P_i$ approximate the unknown Jacobian
matrices $J_i$ up to a rotation. To use them in place of $J_i$, we
must first align the projections correctly. Therefore, our guess for
$A_i$ is of the form  $A_i=P_i O_i$, where $O_i$ are $d \times d$
rotation matrices that minimize
\begin{equation}\label{eq:f}
f(A_1,\ldots ,A_n)=\sum_{i=1}^n \sum_{j\in
\mathrm{Neighbors}(i)}\norm{A_i -A_j}_F^2 \,.
\end{equation}
The rotation matrices $O_i$ can be found using simulated annealing,
as described in Section~\ref{sec:SA}.

Once the matrices $A_i$ are found, we need to minimize
$R(X,Y|A_1,\ldots,A_n)$. We first write $R(X,Y|A_1,\ldots,A_n)$ in
matrix notation as
\begin{equation*}
    R(X,Y|A_1,\ldots,A_n)=\frac{1}{n}\sum_{i=1}^n \mathrm{tr}\left((X-YA_i')'H_i(X-Y A_i')\right)\,.
\end{equation*}
Here $H_i$ is the centering matrix of neighborhood $X_i$
\begin{equation*}
H_i(k,l) = \left\{
\begin{array}{cc}
          1-\frac{1}{k(i)} & k=l \mathrm{\,and\,} k \in \mathrm{Neighbors}(i)  \\
         - \frac{1}{k(i)} &  k \neq l \mathrm{\, and\,} k,l \in \mathrm{Neighbors}(i) \\
          0& \mathrm{elsewhere}\,.
        \end{array} \right.
\end{equation*}

The rows of the matrix $H_i X$ at $\mathrm{Neighbors}(i)$ indices
are $x_i$'s centered neighborhood, where all the other rows equal
zero, and similarly for $H_i Y$. .

Taking the derivative of $R(X,Y|A_1,\ldots,A_n)$
(Eq.~\ref{eq:RXYA1AN}) with respect to $Y$~\citep[see][]{Mardia} and
using the fact that $A_i'A_i=I$, we obtain
\begin{equation}\label{eq:df(Y)}
  \df{Y} R(X,Y|A_1,\ldots,A_n)= \frac{2}{n}  \sum_{i=1}^n H_iXA_i-\frac{2}{n} \sum_{i=1}^n H_i Y \,.
\end{equation}
Using the general inverse of $\sum_{i=1}^n H_i$ we can write $Y$
as
\begin{equation}\label{eq:Y}
Y=\Big(\sum_{i=1}^n H_i\Big)^{\bot}\sum_{i=1}^n H_i X A_i \,.
\end{equation}

Summarizing, we present the PSA algorithm.
\begin{enumerate}
     \item \textbf{Initialization:} \begin{itemize}
        \item Find the neighbors $X_i$ of each point $x_i$.
        \item Find the PCA projection matrices $P_i$ of the neighborhood $X_i$.
    \end{itemize}

    \item \textbf{Alignment of the projection matrices:}\\ Find $A_i$ that
    minimize Eq.~\ref{eq:f}
    using, for example, simulated annealing (see Section~\ref{sec:SA}).
    \item \textbf{Find the embedding:}\\
    Compute $Y$ according to Eq.~\ref{eq:Y}.
\end{enumerate}

The advantage of this algorithm is that it is global. The
computation time of this algorithm (assuming that the matrices $A_i$
are already known) depends on multiplying by the inverse of the
sparse symmetric semi-positive definite matrix $\sum_{i=1}^n H_i $,
which can be very costly. However, this matrix need not be computed
explicitly. Instead, one may solve $d$-linear-equation systems of
the form $(\sum_{i=1}^n H_i)x=b$, which can be computed much faster
\citep[see, for example,][]{SparseMatrices}.

\subsection{Simulated Annealing (SA) alignment procedure}\label{sec:SA}
In step 2.\ of PSA (see Section~\ref{sec:aligmentMinimization}), it
is necessary to align the PCA projection matrices $P_i$. In the
following we suggest an alignment method based on simulated
annealing (SA)~\citep{kirkpatrick}. The aim of the suggested
algorithm is to find a set of columns-orthonormal matrices
$A_1,\ldots,A_n$  that minimize Eq.~\ref{eq:f}. A number of
closely-related algorithms, designed to find embeddings using
alignment of some local dimensionally-reduced descriptions, were
previously suggested. \citet{ProbabilisticGlobalCoordinationRoweis}
and \citet{VerbeekCoordinatingPCA} introduced algorithms based on
probabilistic mixtures of local FA and PCA structures, respectively.
As Eq.~\ref{eq:f} and these two algorithms suffer from local minima,
the use of simulated annealing may be beneficial. Another algorithm,
suggested by~\citet{TehAutomaticAlignment}, uses a convex objective
function to find the alignment. The output matrices of this
algorithm are not necessarily columns-orthonormal, as is required in
our case .

Minimizing Eq.~\ref{eq:f} is similar to the Ising model
problem~\citep[see, for example,][]{IsingModel}. The Ising model
consists of a neighbor-graph and a configuration space that is the
set of all possible assignments of $+1$ or $-1$ to each vertex of
the graph. A low-energy state is one in which neighboring points
have the same sign. Our problem consists of a neighbor-graph with a
configuration space that includes all of the rotations of the
projection matrices $P_i$ at each point $x_i$. Minimizing the
function $f$ is similar to finding a low-energy state of the Ising
model. As solutions to the Ising model usually involve algorithms
such as simulated annealing , we take the same path here.
%
%

We present the SA algorithm, following the algorithm suggested
by~\citet{SA_continuous}, modified for our problem.

\begin{enumerate}

    \item \textbf{Initialization:}
    \begin{itemize}
        \item Choose an initial configuration (for example, using GP).
        \item Compute the initial temperature~\citep[see][]{SA_continuous}.
        \item Define the cooling scheme~\citep[see][]{SA_continuous}.
    \end{itemize}

    \item \textbf{Single SA step:}\begin{itemize}
        \item Choose $i$ randomly.\\
    Generate a random $d \times d$ rotation matrix $O_i$~\citep[see][]{generatingRandomMatrices}.\\
     Define $A_i^{\mathrm{New}} \equiv A_i O_i$.
        \item Compute $f(A_1,\ldots, A_i^{\mathrm{New}},\ldots,A_n)$. \\Note
    that it is enough to compute $\sum_{\mathrm{Neighbors}(i)}\norm{
    A_i^{\mathrm{New}}-A_j}_F^2$.
        \item Accept $A_i^{\mathrm{New}}$ if either
        \begin{equation*}
f(A_1,\ldots, A_i^{\mathrm{New}},\ldots,A_n)<f(A_1,\ldots,
        A_i,\ldots,A_n)
        \end{equation*}
         or with some probability depending on the current
        temperature.

        \item Decrease the temperature and stop if the lowest temperature is reached.
    \end{itemize}
   \item \textbf{Outer iterations:}
   \begin{itemize}
    \item First iteration: Perform a run of SA on all
   matrices $A_1,\ldots,A_n$.
    \\Find the largest cluster of aligned matrices (for example, use BFS~\citep{corman} and define an alignment criterion).
    \item Other iterations: Apply SA to the matrices that are not in the largest cluster. Update the largest cluster after each run.
    \item Repeat until the size of the largest cluster includes almost all of the matrices $A_i$.
   \end{itemize}
\end{enumerate}

Using SA is complicated. The cooling scheme requires the estimation
of many parameters, and the run-time depends heavily on the correct
choice of these parameters. For output of large dimension, the
alignment is difficult, and the output of SA can be poor. Although
SA is a time-consuming algorithm, each iteration is very simple,
involving only $O(Kqd^3)$ operations, where $K$ is the maximum
number of neighbors, and $q$ and $d$ are the input and output
dimensions, respectively. In addition, the memory requirements are
modest. Therefore, SA can run even when the number of points is
large.

\subsection{Iterative procedure}\label{sec:iterative}
Given a solution of GP, PSA, or any other technique, it is usually
possible to modify $Y$ so that $R(X,Y)$ is further decreased. The
idea of the iterative procedure we present here was suggested
independently by~\citet{LTSA}, but no details were supplied.

In Section~\ref{sec:aligmentMinimization}, we showed that given $Y$,
the improved matrices $A_i$ are obtained by finding the Procrustes
matrices between $Y_i$ and $X_i$. Given the matrices $A_i$, the
embedding $Y$ can be found using Eq.~\ref{eq:Y}. An iterative
procedure would require finding first the new matrices $A_i$ and
then a new embedding $Y$ at each stage. This would be repeated until
the change in the value of $R(X,Y)$ was small.

The problem with this scheme is that it involves the computation of
the inverse of the matrix $\sum_{i=1}^{n}H_i$ (see end of
Section~\ref{sec:aligmentMinimization}). We therefore suggest a
modified version of this iterative procedure, which is easier to
compute. Recall that
\begin{equation*}
    R(X,Y)=\sum_{i=1}^n \sum_{j \in \mathrm{Neighbors}(i) } \norm{x_j-A_i y_j-b_i }^2 \,.
\end{equation*}
The least-squares solution for $b_i$ is
\begin{equation}\label{eq:find_bi}
\frac{1}{|\{\mathrm{Neighbors}(i)\}|}\sum_{j \in
\mathrm{Neighbors}(i) } \left(x_j-A_i y_j\right)\,.
\end{equation}
The least-squares solution for $y_j$ is
\begin{equation}\label{eq:find_yj} \frac{1}{|\{i:j\in
\mathrm{Neighbors}(i)\}|}\sum_{\{i:j\in \mathrm{Neighbors}(i)\}}
A_i'(x_j-b_i)\,.
\end{equation}
Note that we get a different solution than that in
Eq.~\ref{eq:df(Y)}. The reason is that here we consider $b_i$ as
constants when we look for a solution for $y_j$. In fact, $y_j$
appear in the definition of the $b_i$. However, as $y_j$ appear
there multiplied by $1/k(i)$, these terms make only a small
contribution.

We suggest performing the iterations as follows. First, find the
Procrustes rotation matrices $A_i$ and the translation vectors $b_i$
using Eq.~\ref{eq:find_bi}. Then find $y_j$ using
Eq.~\ref{eq:find_yj}. Repeat until $R(X,Y)$ no longer decreases
significantly.

%% file: numerical.tex
\section{Numerical Examples}\label{sec:NumericalExamples}
In this section we evaluate the algorithms GP and PSA on data sets
that we assumed to be sampled from underlying manifolds. We compare
the results to those obtained by LLE~\citep{LLE},
Isomap~\citep{ISOMAP}, and LTSA~\citep{LTSA}, both visually and
using the measures $R_N(X,Y)$ and $R_C(X,Y)$
(see~Table~\ref{tb:compare}).

The algorithms GP and PSA were implemented in the Matlab
environment, running on a Pentium 4 with a 3 Ghz CPU and 0.98 GB of
RAM. The alignment stage of PSA was implemented using SA (see
Section~\ref{sec:SA}). The runs of both GP and PSA were followed by
the iterative procedure described in Section~\ref{sec:iterative} to
improve the minimization of $R(X,Y)$. LLE, Isomap, and LTSA were
evaluated using the Matlab code taken from the sites of the
respective authors. The algorithm SDE~\citep{Weinberger}, whose
minimization is closest in spirit to ours, was also tested; however,
it suffers from heavy computational demands, and the results of this
algorithm could not be obtained using the code provided in the site.

The data sets are described in Table~\ref{tb:datasets}. We ran all
five algorithms using $k=6,9,12,15$ and $18$ nearest neighbors.
The minimum values for $R_N(X,Y)$ and $R_C(X,Y)$ are presented in
Table~\ref{tb:compare}. The results in all cases were
qualitatively the same, therefore in the images we show the
results for $k=12$ only.

\begin{table}[!bh]
  \centering

\begin{tabular}{|c|c|c|c|c|l|c|}
\hline
Name & n  & q & d & Description & Figure \\
\hline Swissroll & 1600  &3 & 2 &
isometrically embedded in $R^3$&Fig~\ref{fig:swissroll}\\
\hline Hemisphere & 2500 & 3 &2 & not isometrically embedded & Fig~\ref{fig:sphere}\\
&&&&in $R^3$&\\ \hline Cylinder & 800 & 3&2 & locally isometric to
$R^2$,& Fig~\ref{fig:cylinder}\\  & &&& has no
global embedding in $R^2$ & \\
\hline Faces &1965  &560 &3 & $20\times 28$ pixel grayscale face
images& Fig~\ref{fig:faces}\\ &&&& \citep[see][]{FreyFace}&\\
\hline Twos & 638  & 256 & 10 & images of handwritten ``2''s&
None,due
\\ &&&& from the USPS data set& to output \\
&&&& of handwritten digits~\citep{twoDatabase} & dimension \\
  \hline
\end{tabular}
\bigskip \caption{Description of five data sets used to compare the
different algorithms. $n$ is the sample size and $q$ and $d$ are the
input and output dimensions, respectively.}\label{tb:datasets}
\end{table}
Overall, GP and PSA perform satisfactorily as shown both in the
figures and in Table~\ref{tb:compare}. The fact that in most of the
examples GP and PSA get lower values than LLE, Isomap, and LTSA is
perhaps not surprising, as GP and PSA are designed to minimize
$R(X,Y)$. The run-times of the algorithms excluding PSA is on the
order of seconds while it takes PSA a few hours to run. Memory
requirements of GP and PSA are lower than those of the other
algorithms. As a consequence of the memory requirements, results
could not be obtained for LLE, Isomap and LTSA for $n>3000$.

Use of the measure $R(X,Y)$ allows a quantitative comparison of
visually similar outputs. Regarding the output of the cylinder (see
Fig.~\ref{fig:cylinder}), for example, PSA and Isomap both give
topologically sound results; however, $R(X,Y)$ shows that locally,
PSA does a better job. In addition, $R(X,Y)$ can be used to optimize
embedding parameters such as neighborhood size (see
Fig.\ref{fig:k_LLE}).

\begin{table}[!bh]
\centering
\begin{tabular}{|c|c|c|c|c|c|}
  \hline
& Swissroll  & Hemisphere &Cylinder &Faces&Twos   \\

\hline GP &0.00 [0.00]&0.02 [0.01]&0.13 [0.01]& 0.45 [0.36]&0.00 [0.00]\\
\hline PSA &0.00 [0.00]&0.03 [0.01]&0.02 [0.01]&0.35 [0.30&0.00 [0.00]\\
\hline LLE &0.81 [0.23]&0.60 [0.00]&0.73 [0.13]& 0.99 [0.79]&0.82 [0.23]\\
\hline Isomap &0.01 [0.01]&0.03 [0.02]&0.34 [0.25]& 0.5 [0.38]&0.02 [0.01]\\
\hline LTSA &0.99 [0.22]&0.93 [0.04]&0.59 [0.48]&0.99 [0.53]&0.98 [0.37]\\
\hline Lower Bound &0.00&0.00&0.00&0.11&0.00 \\ \hline
\end{tabular}

\bigskip
\caption{Comparison of the output of the different algorithms using
$R_{N}(X,Y)\;[R_C(X,Y)]$ as the measures of the quality of the
embeddings. These values are the minima of both measures as a
function of neighborhood size $k$, for $k=6,9,12,15,18$. The lower
bound was computed using local PCA at each neighborhood (see
Lemma~\ref{lem:PCAminimzeR}). }\label{tb:compare}
\end{table}
\begin{figure}[!t]
\begin{center}
\includegraphics{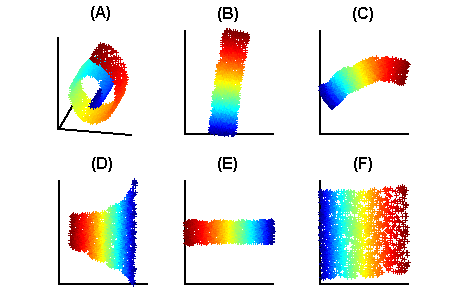}
\caption{A $1600$-point sample taken from the three-dimensional
Swissroll input is presented in~(A). (B)-(F) show the output of GP,
PSA, LLE, Isomap, and LTSA, respectively, for $k=12$. Both GP and
PSA, like Isomap, succeed in finding the proportions of the original
data.} \label{fig:swissroll}
\end{center}
\end{figure}

\begin{figure}[!t]
\begin{center}
\includegraphics{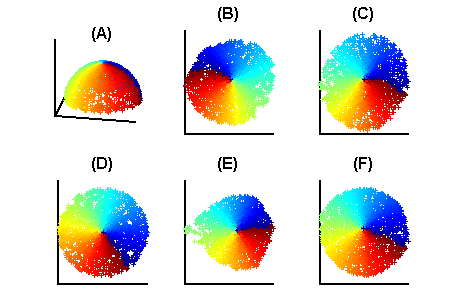}
\caption{\label{fig:sphere}The input of a $2500$-point sample taken
from a hemisphere is presented in~(A). (B)-(F) show the output of
GP, PSA, LLE, Isomap, and LTSA, respectively, for $k=12$. Both GP
and PSA, like the other algorithms, perform the embedding, although
the assumption of isometry does not hold for the hemisphere. }
\end{center}
\end{figure}

\begin{figure}[!t]
\begin{center}
\includegraphics{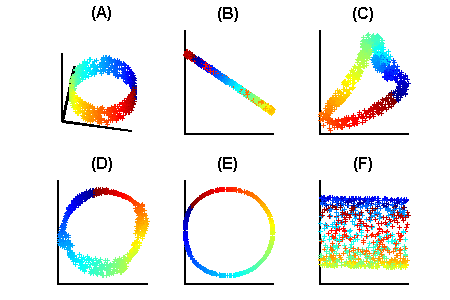}
\caption{The input of an $800$-point sample taken from a cylinder is
presented in~(A). (B)-(F) show the output of GP, PSA, LLE, Isomap,
and LTSA, respectively, for $k=12$. Note that the cylinder has no
embedding in $\R^2$ and it is not clear what is the best embedding
in this case. While PSA, Isomap, and LLE succeeded in finding the
topological ring structure of the cylinder, only PSA and LLE succeed
in preserving the width of the cylinder. GP and LTSA collapse the
cylinder and therefore fail to find the global structure, though
they preform well for most of the neighborhoods (see
Table~\ref{tb:compare}).} \label{fig:cylinder}
\end{center}
\end{figure}

\begin{figure}[!t]
\begin{center}
\includegraphics{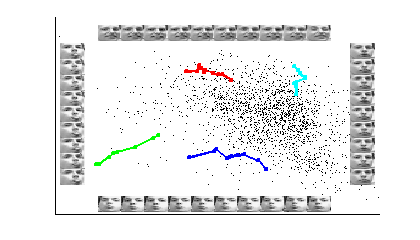}
\caption{The projection of the three-dimensional output, as computed
by PSA, on the first two coordinates (small points). The input used
was a $1965$-point sample of grayscale images of faces (see
Table~\ref{tb:datasets}). The boxes connected by lines are nearby
points in the output set. The images are the corresponding face
images from the input, in the same order. We see that nearby images
in the input space correspond to nearby points in the output
space.}\label{fig:faces}
\end{center}
\end{figure}

\begin{figure}[!ht]
\begin{center}
\includegraphics{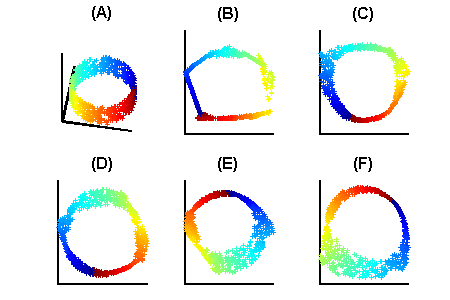}
\caption{The input of an $800$-point sample taken from a cylinder is
presented in~(A). (B)-(F) show the output of LLE for $k=6,9,12,15$
and $18$, respectively. The respective values of $R_C(X,Y)$ are
$0.25,\,  0.15,\,  0.13,\,  0.19,\, 0.17$. While qualitatively the
results are similar, $R_C(X,Y)$ indicates that $k=12$ is optimal.}
\label{fig:k_LLE}
\end{center}
\end{figure}

%% file: discussion.tex
\section{Discussion}

In this section, we emphasize the main results of this work and
indicate possible directions for future research.

We demonstrated that overall, the measure $R(X,Y)$ provides a good
estimation of the quality of the embedding. It allows a
quantitative comparison of the outputs of various embedding
algorithms. Moreover, it is quickly and easily computed. However,
two points should be noted.

First, $R(X,Y)$ measures only the local quality of the embedding.
As emphasized in Fig.~\ref{fig:cylinder}, even outputs that do not
preserve the global structure of the input may yield relatively
low $R$-values if the local neighborhood structure is generally
preserved. This problem is shared by all manifold-embedding
techniques that try to minimize only local attributes of the data.
The problem can be circumvented by adding a penalty for outputs
that embed distant geodesic points close to each other. Distant
geodesic points can be defined, for example, as points at least
$s$-distant on the neighborhood graph, with $s>1$.

Second, $R(X,Y)$ is not an ideal measure of the quality of embedding
for algorithms that normalize their output, such as LLE~\citep{LLE},
Laplacian Eigenmap~\citep{belkin}, and LTSA~\citep{LTSA}. This is
because normalization of the output distorts the structure of the
local neighborhoods and therefore yields high $R$-values even if the
output seems to find the underlying structure of the input. This
point \citep[see also discussion
in][Section~2]{ShaExtensionSpectralMethods} raises the questions,
which qualities are preserved by these algorithms and how can one
quantify these qualities. However, it is clear that these algorithms
do not perform \emph{faithful} embedding in the sense defined in
Section~\ref{sec:Faithful}. The measure $R_C(X,Y)$ addresses this
problem to some degree, by allowing separate scaling of each
neighborhood (see Table~\ref{tb:compare}). One could consider an
even more relaxed measure which allows not only rotation,
translations and scaling but a general linear transformation of each
neighborhood. However, it is not clear what exactly such measure
would quantify. Two new embedding algorithms were introduced. We
discuss some aspects of these algorithms below.

PSA, in the form we suggested in this work, uses SA to align the
tangent subspaces at all points. While PSA works reasonably well for
small input sets and low output dimension spaces, it is not suitable
for large data sets. However, the algorithm should not be rejected
as a whole. Rather, a different or modified technique for subspaces
alignment, for example the use of landmarks~\citep{cIsomap}, is
required in order to make this algorithm truly useful.

GP is very fast ($\OO(n)$ where $n$ is the number of sample
points), can work on very large input sets (even $100,000$ in less
than an hour), and obtains good results as shown both in
Figs.~\ref{fig:swissroll}-\ref{fig:faces} and in
Table~\ref{tb:compare}. This algorithm is therefore an efficient
alternative to the existing algorithms. It may also be used to
choose optimal parameters, such as neighborhood size and output
dimension, before other algorithms are applied. $R(X,Y)$ can be
used for the comparison of GP outputs for varied parameters.

An important issue that was not considered in depth in this paper is
that of noisy input data. The main problem with noisy data is that,
locally, the data seems $q$-dimensional, even if the manifold is
$d$-dimensional, $d<q$. To overcome this problem, one should choose
neighborhoods that are large relative to the magnitude of the noise,
but not too large with respect to the curvature of the manifold. If
the neighborhood size is chosen wisely, both PSA and GP should
overcome the noise and perform the embedding well (see
Fig.~\ref{fig:noise}). This is due to the fact that both algorithms
are based on Procrustes analysis and PCA, which are relatively
robust against noise. Further study is required to define a method
for choosing the optimal neighborhood size.

\begin{figure}[!ht]
\begin{center}
\includegraphics{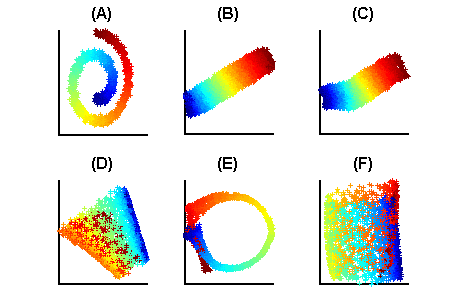}
\caption{The profile of noisy input of a $2500$-point sample taken
from a swissroll is presented in~(A). (B)-(F) show the output of GP,
PSA, LLE, Isomap and LTSA respectively. Note that only GP and PSA
succeed to find the structure of the swissroll.}\label{fig:noise}
\end{center}
\end{figure}


%% file: appendix.tex
\subsection{Proof of Theorem~\ref{thm:R convergence}}\label{sec:proof1}
In this section, we denote the points of neighborhood $X_i$ as
$x_{i_{1}},\ldots ,x_{i_{k(i)}}$, where $k(i)$ is the number of
neighbors in $X_i$.
\begin{proof}
In order to prove that $R(X,Z)=\frac{1}{n}\sum_{i=1}^n G(X_i,Z_i)$
is $\OO(r_{\max}^4)$, it is enough to show that for each $i\in
1,\ldots , n$, $G(X_i,Z_i)=\OO(r_i^4)$, where $r_i$ is the radius of
the $i$-th neighborhood. The proof consists of replacing the
Procrustes rotation matrix $A_i$ by $J_i\equiv J_{\phi}(z_i)$, the
Jacobian of the mapping $\phi$ at $z_i$. Note that the fact that
$\phi$ is an isometry ensures that $J_i'J_i=I$. The Procrustes
translation vector $b_i$ is replaced by $x_i-J_i z_i$.

Recall that by definition $x_j-x_i=\phi(z_j)-\phi(z_i)$; therefore
$x_j-x_i =J_i (z_i-z_j)+\OO\left(\norm{z_j-z_i}^2\right)$. Hence,
\begin{eqnarray}\label{eq:proofR}
G(X_i,Z_i) &=&\inf_{A_i,b_i} \sum_{j=1}^{k(i)} \norm{x_{i_j} -A_i z_{i_j} -b_i}^2\\
&\leq & \sum_{j=1}^{k(i)} \norm{x_{i_j} -J_iz_{i_j} -(x_i-J_i z_i)}^2 \nonumber\\
&=& \sum_{j=1}^{k(i)} \norm{(x_{i_j}-x_i) -J_i(z_{i_j}-z_i)}^2\nonumber\\
&=& \sum_{j=1}^{k(i)}
\OO\left(\norm{z_{i_j}-z_i}^4\right)\,.\nonumber
\end{eqnarray}

$\phi$ is an isometry, therefore
$d_{\mathcal{M}}(x_{i_j},x_i)=\norm{z_{i_j}-z_i}$, where
$d_{\mathcal{M}}$ is the geodesic metric. The sample is assumed to
be dense, hence $\norm{x_{i_j}-x_i} < s_0$, where $s_0$ is the
\emph{minimum branch separation} (see Section~\ref{sec:problem}).
Using~\citet[Lemma~3]{IsoMapConvergence} we conclude that
\begin{equation*}
\norm{z_{i_j}-z_i}=d_{\mathcal{M}}(x_{i_j},x_i)
<\frac{\pi}{2}\norm{x_{i_j}-x_i}\,.
\end{equation*}
We can therefore write $\OO\left(\norm{z_{i_j}-z_i}^4\right)=
\OO(r_i^4)$, which completes the proof.
\end{proof}

\subsection{Proof of Theorem~\ref{thm:RPCA
convergence}}\label{sec:proof2} The proof of Theorem~\ref{thm:RPCA
convergence} is based on the idea that the PCA projection matrix
$P_i$ is usually a good approximation of the span of the Jacobian
$J_i$. The structure of the proof is as follows. First we quantify
the relations between $X_i P_i$ and $X_i J_i$, the projections of
the $i$-th neighborhood using the PCA projection matrix and the
Jacobian, respectively. Then we follow the lines of the proof of
Theorem~\ref{thm:R convergence}, using the bounds obtained
previously.

To compare the PCA projection subspace and tangent subspace at $x_i$
we use the notation of angle between subspaces. Note that both
subspaces are $d$-dimensional and are embedded in the Euclidian
space $\R^q$. The columns of the matrices $P_i$ and $J_i$ consist of
orthonormal bases of the PCA projection space and of the tangent
space, respectively. Denote these subspaces by $\mathcal{P}_i$ and
$\mathcal{J}_i$, respectively. Surprisingly, the angle between
$\mathcal{P}_i$ and $\mathcal{J}_i$ can be arbitrarily large.
However, in the following we show that even if the angle between the
subspaces is large, the projected neighborhoods are close.

We start with some definitions. The \emph{principal angles}
$\sigma_1,\ldots,\sigma_d$ between $\mathcal{J}_i$ and
$\mathcal{P}_i$ may be defined recursively for $p = 1,\dots, d$
as~\citep[see][]{MatrixComputations}
\begin{equation*}
    \cos(\sigma_p) = \max_{v\in \mathcal{P}_i} \max_{w\in \mathcal{J}_i} v'w \,,
\end{equation*}
subject to
\begin{equation*}
\norm{v}=\norm{w}=1,\; v'v_k=0,w'w_k=0\;; k=1,\ldots,p-1\,.
\end{equation*}
The vectors $v_1, \ldots , v_d$ and $w_1,\ldots,w_d$ are called
\emph{principal vectors}.

The fact that $P_i$ and $J_i$ have orthogonal columns leads to a
simple way to calculate  the principal vectors and angles
explicitly. Let $ULV'$ be the svd of $J_i' P_i$.
Then~\citep[see][]{MatrixComputations}
\begin{enumerate}
    \item $v_1, \ldots , v_d$ are given by the columns of $P_iV$.
    \item $w_1, \ldots , w_d$ are given by the columns of $J_iU$.
\end{enumerate}

The relations between the two sets of vectors plays an important
role in our computations. Write $w_p= \ap v_p + \bp \vp$, where
$\ap= w_p 'v_p$, $\bp =\norm{w_p-\ap v_p}$ and $\vp=\frac{w_p-\ap
v_p}{\norm{w_p-\ap v_p}}$. Note that $\ap$ is the cosine of the
$p$-th principal angle between $\mathcal{P}_i$ and $\mathcal{J}_i$.
The angle between the subspaces is defined as $ \arccos(\alpha_d)$
and the distance between the two subspaces is defined to be
$\sin(\alpha_d)$.

We now prove some basic claims related to the principal vectors.

\begin{lemma}\label{lem:principalVectors}
Let $P_i$ be the projection matrix of the neighborhood $X_i$ and let
$J_i$ be the Jacobian of $\phi$ at $z_i$. Let $ULV'$ be the svd of
$J_i' P_i$ and $v_1, \ldots , v_d$ and $w_1,\ldots,w_d$ be the
columns of $P_iV$ and $J_iU$, respectively. Then
\begin{enumerate}
    \item $v_1,\ldots,v_d$ are an orthonormal vector system.
    \item $w_1,\ldots,w_d$ are an orthonormal vector system.
    \item $v_p \perp w_q$ for $p \neq q$.
    \item $\vp \perp v_q^\perp$ for $p \neq q$.
    \item $\vp \perp v_q$ for $q=1,\ldots,d$.
\end{enumerate}
\end{lemma}
\begin{proof}${}$
\begin{enumerate}
    \item True, since $P_i V$  is an orthonormal matrix.
    \item True, since $J_i U$ is an orthonormal matrix.
    \item Note that $(J_i U)'(P_i V)=U'(J_i ' P_i)V=L$ where $L$ is
    a diagonal non-negative matrix.
    \item \begin{eqnarray*}
      (\bp \vp)'(\beta_q v_q^{\perp}) &=& (w_p -\ap v_p)'(w_q -\ap v_q) \\
       &=& w_p'w_q -  v_p{}'w_q -v_q'w_p +v_p{}'v_q =0\,.
    \end{eqnarray*}
    \item\begin{eqnarray*}
      (\bp\vp)' v_q &=& (w_p -\ap v_p)'v_q\\
      &=& w_p'v_q-\ap v_p{}'v_q \\
       &=& \delta_{pq}\ap-\delta_{pq}\ap=0\,.
    \end{eqnarray*}
\end{enumerate}
\end{proof}

Using the relation between the principal vectors, we can compare the
description of the neighborhood $X_i$ in the local PCA coordinations
and its description in the tangent space coordinations. Here we need
to exploit two main facts. The first fact is that the local PCA
projection of a neighborhood is the best approximation, in the $l_2$
sense, to the original neighborhood. Specifically, it is a better
approximation than the tangent space in the $l_2$ sense. The second
is that in a small neighborhood of $x_i$, the tangent space itself
is a good approximation to the original neighborhood. Formally, the
first assertion means that
\begin{equation}\label{eq:XgeqPgeqJ}
    \sjk \norm{\xj}^2 \geq \sjk \norm{P_i'\xj}^2  \geq \sjk \norm{J_i'\xj}^2
\end{equation}
while the second assertion means that
\begin{equation}\label{eq:XminusJisO4}
    \sjk \norm{\xj}^2 -\sjk \norm{J_i'\xj}^2= \OO(r_i^4)\,.
\end{equation}

The proof of Eq.~\ref{eq:XminusJisO4} is straightforward. First note
that
\begin{eqnarray*} \label{eqn:relation xj and sj}
  \xj  &=& (x_{i_j}-x_i)-(\bar{x}_i-x_i) \\
   &=&   J_i(z_{i_j}-z_i)-J_i(\bar{z}_i-z_i)+\OO(r_i^2)\\
   &=& J_i \zj + \OO(r_i^2)\,.
\end{eqnarray*}
Hence
\begin{eqnarray*}
  \norm{\xj}^2 -\norm{J_i'\xj}^2& =& \sum_{p=d+1}^q (w_p'\xj)^2 \\
   &=&  \norm{\xj -J_iJ_i'\xj}^2\\
   &=& \norm{J_i \zj -J_iJ_i'(J_i \zj)+ \OO(r_i^2)}^2\\
   &=&\norm{\OO(r_i^2)}^2=\OO(r_i^4)\,,
\end{eqnarray*}
where $w_{d+1},\ldots,w_q$ are a completion of  $w_1,\ldots,w_d$ to
an orthonormal basis of $\R^q$ and we used the fact that $J_i
'J_i=I$.

The following is a lemma regarding the relations between the PCA
projection matrix and the Jacobian projection. It is a consequence
of Eq.~\ref{eq:XgeqPgeqJ}.
\begin{lemma}\label{lem: xj geq p geq j} ${}$
\begin{enumerate}
    \item $\sjk \norm{\xj}^2- \sjk \norm{V'P_i'\xj}^2=\OO(r_i^4)$.
    \item $\sjk \norm{V'P_i'\xj}^2 -\sjk \norm{U'J_i'\xj}^2=\OO(r_i^4)$.
    \item $\xj'v_p=\OO(r_i)$.
    \item $\xj'\vp=\OO(r_i^2)$.
\end{enumerate}
\end{lemma}
\begin{proof}${}$\\
\begin{enumerate}
    \item \addtocounter{enumi}{1}and 2. follow from Eqs.~\ref{eq:XgeqPgeqJ} and~\ref{eq:XminusJisO4}.
    \item follows from the definition of $r_i$.
    \item is a consequence of 1., indeed,
\begin{eqnarray*}
 \sjk \spd \left(\vp {}'\xj \right)^2  &\leq & \sjk \norm{\xj}^2 -\sjk \norm{V'P_i '\xj}^2\\
& = &\OO(r_i^4 )\,.
\end{eqnarray*}
\end{enumerate}
\end{proof}

We now prove Theorem~\ref{thm:RPCA convergence}. Similarly to the
proof of Theorem~\ref{thm:R convergence}, it is enough to show that
$G(X_i P_i,Z_i)=\OO(r_i^3)$.
\begin{eqnarray*}
  G(X_i P_i,Z_i) &=& \inf_{A_i,b_i} \sjk \norm{P_i' x_{i_j} -A_i z_{i_j} -b_i}^2 \\
&\leq & \sjk \norm{P_i' \xj - O_i \zj}^2 \\
   &= & \sjk \norm{P_i' \xj - O_i J_i' \xj +O_i J_i'\xj -O_i \zj}^2 \\
   &\leq & \sjk \norm{P_i' \xj -O_i J_i' \xj }^2 + \sjk \norm{O_i J_i'\xj -O_i \zj}^2\\
    &\equiv& \mathrm{Exp1}+ \mathrm{Exp2}\,,
\end{eqnarray*}
where $O_i$ is some $d \times d$ orthogonal matrix. Note that due to
its orthogonality, Exp2 is independent of the specific choice of
$O_i$.

We choose $O_i=VU'$. Rewriting $\mathrm{Exp1}$ we obtain
\begin{eqnarray*}
   \mathrm{Exp1}&=&\sjk \norm{P_i' \xj -O_i  J_i'\xj }^2  \\
   &=&\sjk \norm{P_i' \xj -VU' J_i'\xj }^2  \\
   &=&\sjk \norm{V'P_i' \xj -(V'V)U' J_i'\xj }^2  \\
   &=&\sjk \spd \left(v_p{}'\xj-w_p '\xj\right)^2 \,.
\end{eqnarray*}
This last expression brings out the difference between the
description of the neighborhood $X_i$ in the local PCA coordinations
and its description in the tangent space coordinates. Using
Lemma~\ref{lem:principalVectors}, we can write
\begin{eqnarray}\label{eq:Ex1in3parts}
   \mathrm{Exp1}&=&\sjk \spd \big((v_p -w_p)'\xj\big)^2\\
   &=& \sjk\spd \big(v_p-\ap v_p-\bp \vp)'\xj\big)^2 \nonumber\\
   &=& \sjk\spd\ (1-\ap)^2\big(v_p{}'\xj \big)^2 \nonumber \\
   &&  - \sjk \spd 2(1-\ap)\bp\big(v_p{}'\xj\big)\big(\vp{}'\xj\big) \nonumber\\
   && + \sjk \spd \bp^2\big(\vp{}'\xj\big)^2\,. \nonumber
\end{eqnarray}
We will use Lemma~\ref{lem: xj geq p geq j} to bound the first
expression of the RHS.
\begin{eqnarray*}
  \OO(r_i^4) &=& \sjk \norm{V'P_i'\xj}^2 -\sjk \norm{U'J_i'\xj}^2 \\
  &=& \sjk\spd \Big\{\big(v_p{}'\xj\big)^2-\big((\ap v_p+\bp \vp)'\xj\big)^2\Big\} \\
  &=& \sjk \spd (1-\ap^2)\big(v_p{}'\xj\big)^2\\
  &&+ 2\sjk\spd\ap\bp \big(v_p{}'\xj\big)\big(\vp{}'\xj\big)\\
  && -\sjk\spd \bp^2 \big(\vp{}'\xj\big)^2\,.
\end{eqnarray*}
Note also that $(1-\ap)^2\leq 1- \ap^2$. Hence,
\begin{eqnarray*}
  \sjk\spd\ (1-\ap)^2\big(v_p{}'\xj \big)^2  &\leq& \sjk\spd \bp^2 \big(\vp{}'\xj\big)^2  \\
  &&- 2\sjk\spd\ap\bp \big(v_p{}'\xj\big)\big(\vp{}'\xj\big)\\
  && +\OO(r_i^4)\,.
\end{eqnarray*}
Plugging it into Eq.~\ref{eq:Ex1in3parts} we get
\begin{eqnarray*}
   \mathrm{Exp1}   &\leq&  \sjk \spd 2\bp^2 \big(\vp{}'\xj\big)^2\\
   && - \sjk \spd 2\bp \big(v_p{}'\xj\big)\big(\vp{}'\xj\big)   +\OO(r_i^4)\\
   &\leq& \OO(r_i^4)+\OO(r_i^3)+\OO(r_i^4)=\OO(r_i^3)\,,
\end{eqnarray*}
where the last inequality is due to Lemma~\ref{lem: xj geq p geq j}.

Proving that Exp2 is $\OO(r_i^4)$ is straightforward.
\begin{eqnarray*}
 \mathrm{Exp2} &=&\sjk \norm{O_i J_i'\xj -O_i \zj}^2  \\
   &=&  \sjk \norm{J_i'\xj - \zj}^2 =\OO(r_i^4)\,,
\end{eqnarray*}
which concludes the proof of Theorem~\ref{thm:RPCA convergence}.

\subsection{Proof of Theorem~\ref{thm:RC convergence}}\label{sec:proof3}
The proof is similar to the proof of Theorem~\ref{thm:R convergence}
(see Section~\ref{sec:proof1}). The proof consists of replacing the
Procrustes rotation matrix $A_i$ and the constant $c_i$ by
$J_i\equiv J_{\tilde{\phi}}(z_i)$, the Jacobian of the mapping
$\tilde{\phi}$ at $z_i$. Note that as $\tilde{\phi}$ is an conformal
mapping which ensures that $J_i'J_i=cI$. The Procrustes translation
vector $b_i$ is replaced by $x_i-J_i z_i$.

Recall that by definition
$x_j-x_i=\tilde{\phi}(z_j)-\tilde{\phi}(z_i)$; therefore $x_j-x_i
=J_i (z_i-z_j)+\OO\left(\norm{z_j-z_i}^2\right)$. Hence,
\begin{eqnarray}\label{eq:proofRC}
G(X_i,Z_i) &=&\inf_{A_i,b_i,c_i} \sum_{j=1}^{k(i)} \norm{x_{i_j} -c_i A_i z_{i_j} -b_i}^2\\
&\leq & \sum_{j=1}^{k(i)} \norm{x_{i_j} -J_iz_{i_j} -(x_i-J_i z_i)}^2 \nonumber\\
&=& \sum_{j=1}^{k(i)}
\OO\left(\norm{z_{i_j}-z_i}^4\right)\,.\nonumber
\end{eqnarray}

As $\tilde{\phi}$ is an conformal mapping, we have $
c_{\min}\norm{z_{i_j}-z_i} \leq d_{\mathcal{M}}(x_{i_j},x_i)$, where
$d_{\mathcal{M}}$ is the geodesic metric and $c_{\min}>0$ is the
minimum of the scale function $c(z)$ measures the scaling change of
$\phi$ at $z$ . The minimum $c_{\min}$ is attained as $\mathcal{D}$
is compact. The last inequality holds true since the geodesic
distance $d_{\mathcal{M}}(x_{i_j},x_i)$ equals to the integral over
$c(z)$ for some path between $z_{i_j}$ and $z_i$.

The sample is assumed to be dense, hence $\norm{x_{i_j}-x_i} < s_0$,
where $s_0$ is the \emph{minimum branch separation} (see
Section~\ref{sec:problem}). Using
again~\citet[Lemma~3]{IsoMapConvergence} we conclude that
\begin{equation*}
\norm{z_{i_j}-z_i}\leq
\frac{1}{c_{\min}}d_{\mathcal{M}}(x_{i_j},x_i)
<\frac{\pi}{2c_{\min}}\norm{x_{i_j}-x_i}\,.
\end{equation*}
We can therefore write $\OO\left(\norm{z_{i_j}-z_i}^4\right)=
\OO(r_i^4)$. Dividing by the normalization $\norm{HX_i}_F^2 $ for
each neighborhood we obtain $R_{C}(X,Y)=\OO(r_{\max}^2)$ which
completes the proof.